\documentclass{article}
\usepackage{graphicx} 
\usepackage{multirow}
\usepackage{tabularx}
\usepackage[utf8]{inputenc}
\usepackage{fullpage}
\usepackage{amsmath}
\usepackage{amsthm}
\theoremstyle{plain}
\usepackage{amssymb}
\usepackage[ruled,vlined]{algorithm2e}
\usepackage{xcolor}
\usepackage{hyperref}
\usepackage{mathtools}
\usepackage{mdframed}

\newtheorem{claim}{Claim}
\newtheorem{lemma}{Lemma}
\newtheorem{theorem}{Theorem}
\newtheorem{definition}{Definition}

\newtheorem{fact}{Fact}
\newtheorem{corollary}{Corollary}
\newtheorem{observation}{Observation}

\newtheorem{assumption}{Assumption}
\newtheorem{example}{Example}

\newcommand{\remove}[1]{}

\newcommand{\calX}{{\cal X}}
\newcommand{\calY}{{\cal Y}}

\title{An $\tilde{\mbox{O}}$ptimal Differentially Private PAC Learner for Concept Classes with VC Dimension 1}
\author{Chao Yan\thanks{Department of Computer Science, Georgetown University. {\tt cy399@georgetown.edu}. Research is supported in part by a gift from Georgetown University.}}
\date{}

\begin{document}

\maketitle

\begin{abstract}
    We present the first nearly optimal differentially private PAC learner for any concept class with VC dimension 1 and Littlestone dimension $d$. Our algorithm achieves the sample complexity of $\tilde{O}_{\varepsilon,\delta,\alpha,\beta}(\log^*d)$, nearly matching the lower bound of $\Omega(\log^*d)$ proved by Alon et al.~\cite{AlonLMM19} [STOC19]. Prior to our work, the best known upper bound is $\tilde{O}(VC\cdot d^5)$ for general VC classes, as shown by Ghazi et al.~\cite{ghazi2021sample} [STOC21]. 

\end{abstract}

\section{Introduction}
Machine learning algorithms can access sensitive information from the training dataset. We research the privacy-preserving machine learning technique, introduced by Kasiviswanathan et al.~\cite{KLNRS08}, that targets to learn a hypothesis while preserving the privacy of individual entries in the dataset. Informally, the goal is to construct a learner that satisfies the requirements of probably approximately correct (PAC) learning~\cite{Valiant84} and, simultaneously, differential privacy~\cite{DMNS06}.

\remove{
\begin{definition}[PAC Learning~\cite{Valiant84}]
    Given a set of $n$ points $S=(\mathcal{X}\times\{0,1\})^n$ drawn i.i.d. from an unknown distribution $\mathcal{D}$ and labels that are given by an unknown concept $c\in\mathcal{C}$, we say a learner $L$ (possibly randomized) $(\alpha,\beta)$-PAC learns $\mathcal{C}$ if $h=L(S)$ and 
    $$
    \Pr[error_{\mathcal{D}}(c,h)\leq\alpha]\geq1-\beta.
    $$
\end{definition}
}

\begin{definition}[Differential Privacy~\cite{DMNS06}]
    A randomized algorithm $M$ is called $(\varepsilon,\delta)$-differentially private if for any two dataset $S$ and $S'$ that differ on one entry and any event $E$, it holds that
    $$
    \Pr[M(S)\in E]\leq e^{\varepsilon}\cdot \Pr[M(S')\in E]+\delta.
    $$
    Specifically, when $\delta=0$, we call it pure-differential privacy. When $\delta>0$, we call it approximate-differential privacy
\end{definition}

\remove{
\begin{definition}[differentially private learning~\cite{KLNRS08}]
    We say a learner $L$ $(\alpha,\beta, \varepsilon,\delta)$-differentially privately PAC learns the concept class $\mathcal{C}$ if
    \begin{enumerate}
        \item $L$ $(\alpha,\beta)$-PAC learns $\mathcal{C}$.
        \item $L$ is $(\varepsilon,\delta)$-differentially private.
    \end{enumerate}
\end{definition}
}

In the PAC learning task, the learner is given a set of $n$ points $S=(\mathcal{X}\times\{0,1\})^n$ drawn i.i.d. from an unknown distribution $\mathcal{D}$ and labels that are given by an unknown concept $c:\mathcal{X}\rightarrow\{0,1\}$, where $c\in\mathcal{C}$ and $\mathcal{C}$ is the concept class. The goal is to output a hypothesis $h:\mathcal{X}\rightarrow\{0,1\}$ such that $\Pr_{x\sim \mathcal{D}}[c(x)\neq h(x)]\leq\alpha$, where $\alpha$ is the parameter of accuray.
We call the required dataset size of the learning task the $\emph{sample complexity}$, which is a fundamental question in learning theory. For non-private PAC learning, it is well-known that the sample complexity is linear to the VC dimension of the concept class~\cite{Valiant84}. However, the sample complexity is less well understood in the differentially private setting. Kasiviswanathan et al.~\cite{KLNRS08} give a upper bound $O(\log|\mathcal{C}|)$, which works for pure differential privacy ($\delta=0$).  

In the approximate privacy regime ($\delta>0$), several works~\cite{BNS13b,BNSV15,beimel2019private,kaplan2020private,KaplanLMNS19,CohenLNSS22,NTY25} show that the sample complexity can be significantly lower than that in the pure setting. Alon et al.~\cite{AlonLMM19} and Bun et al.~\cite{bun2020equivalence} find that the sample complexity of approximately differentially private learning is related to the Littlestone dimension~\cite{littlestone87} of the concept class $\mathcal{C}$, which characterizes the mistakes bound of online learning~\cite{littlestone87}. In detail, for a concept class $\mathcal{C}$ with Littlestone dimension $d$, Alon et al.~\cite{AlonLMM19} prove a lower bound $\Omega(\log^*d)$\footnote{$\log^*d$ is the iterated logarithm of $d$} and Bun et al.~\cite{bun2020equivalence} provide an upper bound $O(2^{O(2^d)})$. Subsequently, Ghazi et al.~\cite{ghazi2021sample} improve the upper bound to $\tilde{O}(VC\cdot d^5)$.

This leaves a large gap between the lower bound $\Omega(\log^*d)$ and the upper bound $\tilde{O}(VC\cdot d^5)$, even when the VC dimension is as small as 1. The main question is: what is the best dependence on $d$? One important example is the halfspaces class, which can be privately learned with $poly(VC,\log^*d)$ examples by the work of Nissim et al.~\cite{NTY25}. So it's natural to conjecture that the tight sample complexity of differentially private learning depends on $\log^*d$.

Moreover, the work by Nissim et al.~\cite{NTY25} also shows an upper bound of $\tilde{O}(\log^*|\mathcal{X}|)$ (recall $|\mathcal{X}|$ is the input domain size) for any concept class with VC dimension 1. However, $|\mathcal{X}|$ could be arbitrarily larger than $d$ for the general concept class, and a significant gap remained between $\log^*d$ and $\min\{d^5,\log^*|\mathcal{X}|\}$. The remaining question is:

\begin{mdframed}
    Could we privately learn any VC 1 class with sample size $poly(\log^*d)$? 
\end{mdframed}

In this work, we give a positive answer to this question and give a nearly tight bound $\tilde{\theta}_{\alpha,\beta,\varepsilon,\delta}(\log^*d)$.

\subsection{Our Result}

\begin{theorem} [Improper Learning]
    For any concept class $\mathcal{C}$ with VC dimension 1 and Littlestone dimension $d$, there is an $(\varepsilon,\delta)$-differentially private algorithm that $(\alpha,\beta)$-PAC learns $\mathcal{C}$ if the given labeled dataset has size
    $$
    N\geq O\left(\frac{\log^*d\cdot\log^2(\frac{\log^*d}{\varepsilon\beta\delta})}{\varepsilon}\cdot \frac{48}{\alpha}\left(10\log(\frac{48e}{\alpha})+\log(\frac{5}{\beta})\right)\right)=\tilde{O}_{\beta,\delta}\left(\frac{\log^*d}{\alpha\varepsilon}\right)
    $$
    
\end{theorem}

\begin{theorem}[Proper Learning]
    For any concept class $\mathcal{C}$ with VC dimension 1 and Littlestone dimension $d$, and given labeled dataset with size
    $$
    \begin{array}{rl}
        N&\geq O\left(\frac{\log^*d\cdot\log^2(\frac{\log^*d}{\varepsilon\beta\delta})}{\varepsilon}\cdot \frac{48}{\alpha}\left(10\log(\frac{48e}{\alpha})+\log(\frac{5}{\beta})\right)+\frac{(\log(1/\alpha)+\log(1/\alpha\beta))\cdot\sqrt{\log(1/\delta)}}{\alpha^{2.5}\varepsilon}\right)\\
        &=\tilde{O}_{\beta,\delta}\left(\frac{\log^*d}{\alpha\varepsilon}+\frac{1}{\alpha^{2.5}\varepsilon}\right)
    \end{array}    
    $$
    there is an $(\varepsilon,\delta)$-differentially private algorithm that properly $(\alpha,\beta)$-PAC learns $\mathcal{C}$.
\end{theorem}

\subsection{Other Related Works}

\paragraph {Two Well Researched VC 1 Classes} The first class is the point functions class (also called the indicator function). It is defined to be $f_t(x)=1$ if $x=t$ and $f_t(x)=0$ if $x\neq t$. The point functions class has Littlestone Dimension 1, and Beimel et al.~\cite{BNS13b} prove it can be privately learned with $O(1)$ examples using the choosing mechanism (Lemma~\ref{lem:choosing mechanism}). 

The second example is the thresholds class. It is defined to be $f_t(x)=1$ if $x\geq t$ and $f_t(x)=0$ if $x< t$. The thresholds class has Littlestone dimension $\log|\mathcal{X}|$. There is a series of works on privately learning thresholds~\cite{BNS13b, BNSV15,KaplanLMNS19,CohenLNSS22}, and finally been found that we can privately learn thresholds with $\tilde{O}(\log^*|\mathcal{X}|)$ examples. The key observation is the equivalence between private learning thresholds and private median (or private interior point), which is observed by Bun et al.~\cite{BNSV15}.

\paragraph{Proper and Improper Private Learning} We call a learner proper if the output hypothesis is in the concept class and improper if the output may not be in the concept class. For pure differential privacy, Beimel et al.~\cite{BKN10} prove a lower bound on the sample complexity of proper learning point functions is $\Omega(\log|X|)$ and an upper bound of improper learning point functions $O(1)$. It shows a significant separation between proper and improper learning with pure differential privacy. Our result shows that for any VC 1 class, the separation between proper and improper learning with approximate differential privacy cannot be larger than $O(\frac{1}{\alpha^{O(1)}})$ (ignoring other parameters).

\subsection{Overview of Technique}
\subsubsection{Improper Learning}
The key observation is the tree structure\footnote{The Tree structure of the VC 1 class is mentioned in a lot of papers and is usually researched in the one-inclusion graph context. But most of them only consider the maximum concept classes. The observation by Ben-David can include non-maximum classes.} and partial order in the VC 1 class, which is observed by Ben-David~\cite{BenDavid20152NO}. In the tree structure, each node is a point from the domain $\mathcal{X}$, and each hypothesis is a path from one node to the tree's root. Another important idea is that we focus on the thresholds dimension instead of the Littlestone dimension. We show that the tree's height is upper bounded by the threshold dimension of the concept class, which has an upper bound $O(2^d)$~\cite{shelah1990classification,hodges1997shorter,AlonLMM19}. So that we can apply the private median algorithm (Theorem~\ref{thm:private median}) and achieve $\tilde{O}(\log^*d)$ sample complexity.

Based on the tree structure, we show that every VC 1 class can be privately learned in two steps:

\paragraph{Step 1: privately find a "good length".} We use the partition and aggregate method to construct the private learner. The labeled dataset $S$ is randomly partitioned into subsets $S_1,\dots,S_t$, where $t=\tilde{O}(\log^*d)$. Each subset has a set of "deterministic points", whose labels are fixed to be 1 by the given labeled points. We show that the deterministic points can be used to construct an accurate hypothesis, and all the sets of deterministic points of $S_1,\dots,S_t$ are on the same path. Since the length of the path is at most $O(2^d)$, we use the private median algorithm (Theorem~\ref{thm:private median}~\cite{CohenLNSS22}) to select a "good length" with sample size $\tilde{O}(\log^*d)$. 

\paragraph{Step 2: privately find a "good path".} We let the "good path" be the path such that all points on this path are deterministic by the given subset. We can show that for all paths with the "good length", only one can be the required "good path". So that we transform it to be a 1-bounded function (Definition~\ref{def: k-bounded}) and we can use the choosing mechanism (Lemma~\ref{lem:choosing mechanism}) to select a "good path". This "good path" can be projected to an accurate hypothesis.

\subsubsection{Example}
\begin{figure}[ht]
\includegraphics[scale=.3]{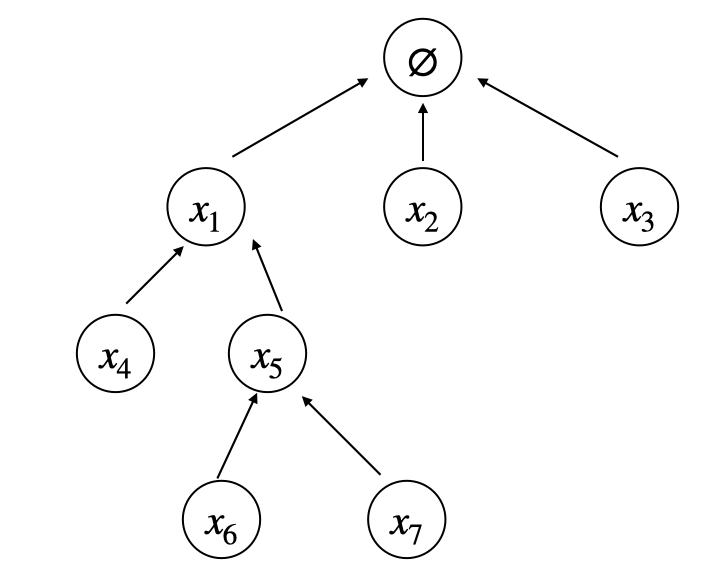}
\caption{Example of tree structure \label{fig:example}}
\end{figure}

Here we give an example. Let $\mathcal{X}=\{x_1,x_2,x_3,x_4,x_5,x_6,x_7\}$. Let $I(f)=\{x:f(x)=1\}$. Then $\mathcal{H}=\{h_1,h_2,h_3,h_4,h_5,h_6,h_7,h_8\}$, where 
$I(h_1)=\{x_1\}$, $I(h_2)=\{x_2\}$, $I(h_3)=\{x_3\}$, $I(h_4)=\{x_1,x_4\}$, $I(h_5)=\{x_1,x_5\}$, $I(h_6)=\{x_1,x_5,x_6\}$, $I(h_7)=\{x_1,x_5,x_7\}$, $I(h_8)=\emptyset$. The tree structure is shown in Figure~\ref{fig:example}. In the tree structure, each concept can be represented by a path from one node to the root (For example, $h_5$ is the path $x_5\rightarrow x_1\rightarrow\emptyset$. That means $h_5$ gives $x_1,x_5$ label 1 and gives other points label 0.). The tree structure has four layers. The first layer has $\emptyset$, the second layer has $x_1,x_2,x_3$, the third layer has $x_4,x_5$, the fourth layer has $x_6,x_7$. We will show that the layer number cannot be "too large" if the Littlestone dimension is bounded.

Let the underlying concept be $h_7$, for all subsets of the given dataset, the deterministic points set can only be $\emptyset$ or $\{x_1\}$ or $\{x_1,x_5\}$ or $\{x_1,x_5,x_7\}$. For instance, if we have $S_1,S_2$ and their deterministic points are $\{x_1\}$ and $\{x_1,x_5,x_7\}$. We record the maximum number of layers of deterministic points, which are 2 and 4. We can privately output the median of the layer numbers, say it is 3. Then we only consider the points on the third layer, which are $x_4$ and $x_5$. Notice that $x_4$ will never be the deterministic point, and approximately half of the subsets will make $x_5$ the deterministic point (because layer 3 is the median layer). Then we can use the choosing mechanism to select $x_5$. We consider the corresponding path to the root and its hypothesis $h_5$. Notice that if $S_1$ and $S_2$ contain enough points, we can show that $h_1$ and $h_7$ have high accuracy by VC theory. Then we can show $h_5$ also has high accuracy because it is in the middle of $h_1$ and $h_7$ (see Figure~\ref{fig:example output}).

\begin{figure}[ht]
\includegraphics[scale=.3]{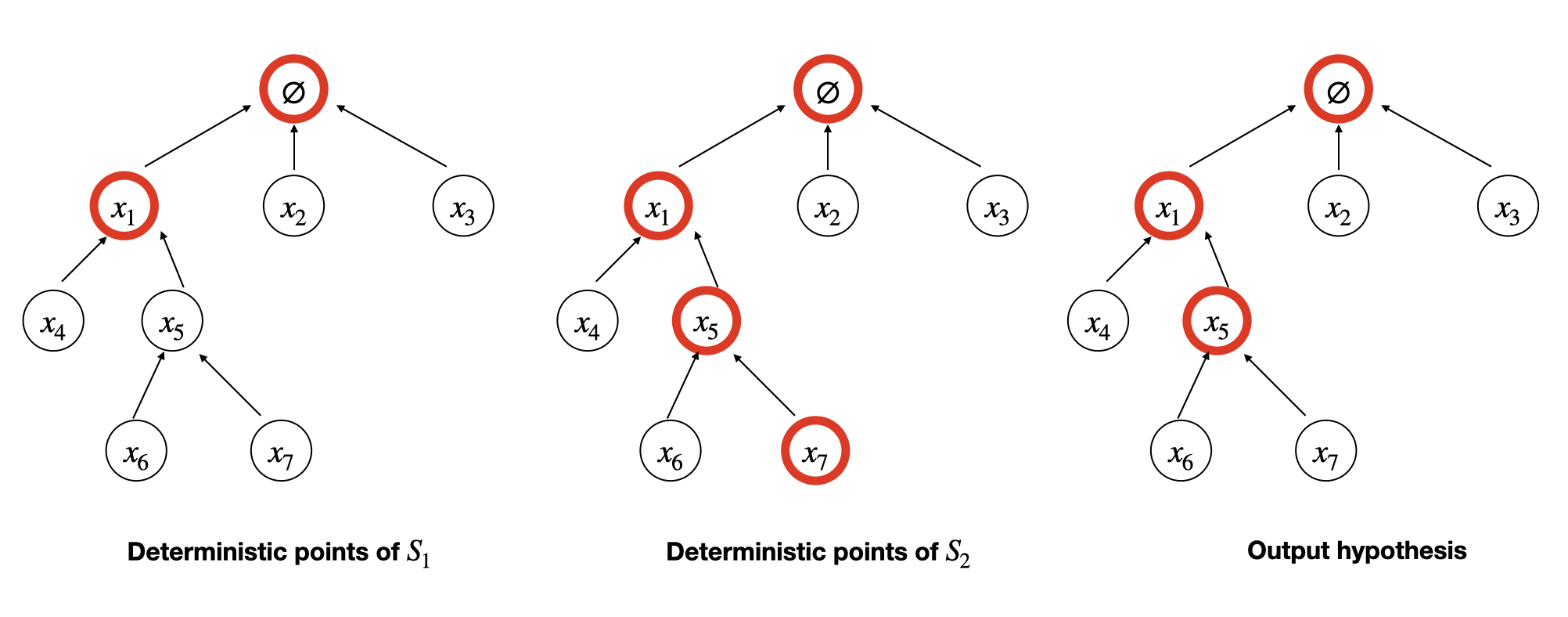}
\caption{Example of output hypothesis \label{fig:example output}}
\end{figure}

\remove{
\subsection{Open Questions}
\paragraph{Extension to Higher VC Dimension} In this work, we show that the sample complexity of privately learning VC 1 class is $\tilde{\theta}(\log^* d)$. We still don't know what the best dependence on $d$ is for $VC\geq 2$.

\paragraph{Proper and Improper Learning} Our learner is proper for maximum VC 1 classes but improper for non-maximum classes. We don't know if that upper bound is also true for proper learning or if there is a separation between proper and improper learning.
}

\subsubsection{Proper Learning}
The key observation is that the improper learner (Algorithm~\ref{alg:improperlearner}) outputs $x_{good}$, which is on the path of the true concept path. Let $x_{c^*_f}$ be the corresponding node of the true concept. Then, for all nodes on the path from $x_{good}$ to $x_{c^*_f}$, they can project to a highly accurate hypothesis. We call a node a proper node if it can be projected to a concept in $\mathcal{C}$. Thus, we can only consider the sub-tree $T_{good}$ that satisfies (see example in Figure~\ref{fig:subtree}):
\begin{enumerate}
    \item the root of $T_{good}$ is $x_{good}$
    \item all leaves of $T_{good}$ are proper nodes
    \item all nodes except leaves are not proper nodes
\end{enumerate}

\begin{figure}[ht]
\includegraphics[scale=.3]{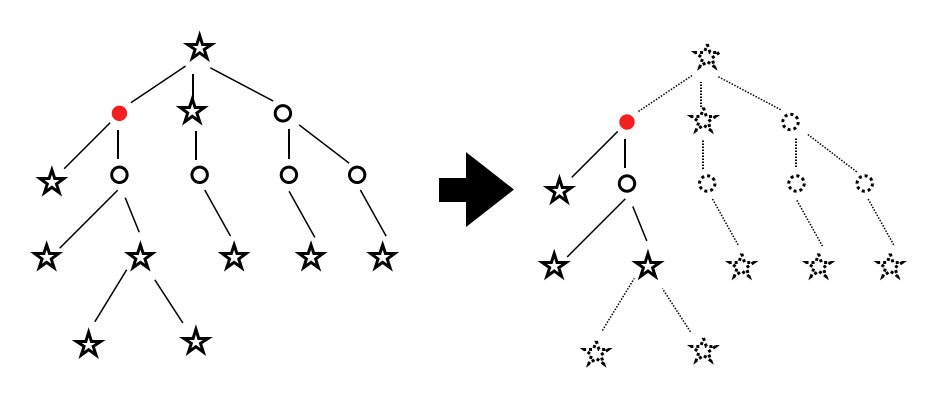}
\caption{We use the circle to represent the improper node and the star to represent the proper node. Let the red node be $x_{good}$. The figure shows the subtree with root $x_{good}$.\label{fig:subtree}}
\end{figure}

By the above observation, there must exist one leaf on the path from $x_{good}$ to $x_{c^*_f}$, and thus it has high accuracy and can be projected to a concept in $\mathcal{C}$. By the analysis of the improper learner, for all nodes on the sub-tree with root $x_{good}$, they correctly classify most points with label 1. Thus, we only need to consider the accuracy on the points with label 0. If we can find a leaf $x'$ such that only a small number of points with label 0 are located on the path from $x_{good}$ to $x'$, then for the corresponding hypothesis $h'$, where $I(h')=\{x|x'\preceq x\}$, $h'$ and $h_{good}$ (recall that $h_{good}$ is corresponding hypothesis of $x_{good}$) disagree a small number of labeled points in $S$. Thus, $h'$ has high accuracy.

Let $v(x)$ be the number of points with label 0 on the path from $x$ to $x_{good}$. The target is to find a proper node $x$ with a small value $v(x)$. Let $w(x)$ be the number of points with label 0 below $x$ on $T_{good}$. Notice that for all $x'\prec x$, we have $v(x')\leq v(x)+\sum_{x'\in T_x}[w(x)]$, where $T_x$ is sub-tree in $T_{good}$ with root $x$. So that we can look for a point $x$ such that $v(x)+\sum_{x'\in T_x}[w(x')]$ is small. Then, all leaves of $T_x$ are proper nodes with high accuracy. Consider another dataset with size $N$ and the following strategy to look for it:
\begin{enumerate}
    \item we set $x_{flag}=x_{good}$ at the beginning, let $X_{children}$ be the set of all children of $x_{flag}$.

    \item Case 1: there exists $x\in X_{children}$ such that $w(x)\leq O(\alpha\cdot N)$, then all leaves $\ell$ of sub-tree $T_x$ have $v(\ell)\leq w(x)+v(x)$. In this case, we cannot have too many "bad children" with heavy weights because every pair of children does not have an intersection part on their subtree, and the sum of weights cannot be larger than the number of points. Running the exponential mechanism can privately select a child with light weight even though $|X_{children}|$ may be large.

    \item Case 2: for any $x\in X_{children}$, we have $w(x)>\Omega(\alpha\cdot N)$. Since $\sum_{x\in X_{children}}w(x)\leq N$, we have $|X_{children}|\leq O(\frac{1}{\alpha})$. So we can use the exponential mechanism to select a child, such that it contains a "good leaf" $\ell$ satisfying $v(\ell)$ is small. Since $|X_{children}|\leq O(\frac{1}{\alpha})$, the error is bounded by $O(\frac{\log(1/\alpha)}{\varepsilon})$, where $\varepsilon$ is the privacy parameter.
\end{enumerate}

In case 1, we immediately find the good leaf. In case 2, we remove at least $\alpha\cdot N$ points that are located in the subtree of the other children. Thus, we can stop in $O(\frac{1}{\alpha})$ steps and the accumulated error is bounded by $O(\frac{\log(1/\alpha)}{\alpha\varepsilon})$. That is, in each iteration, there must exist a leaf $\ell$ in the subtree $T_{x_{flag}}$ satisfying $v(\ell)\leq \min_{\ell'\in T_{good}}v(\ell')+O(\frac{\log(1/\alpha)}{\alpha\varepsilon})$. Since running case 1 causes error $O(\alpha\cdot N)$ and the improper learner guarantees that $\min_{\ell'\in T_{good}}v(\ell')\leq O(\alpha\cdot N)$, the above strategy will make $v(\ell)\leq O(\alpha\cdot N+\frac{\log(1/\alpha)}{\alpha\varepsilon})$.

\section{Notations}
In learning theory, we call $\mathcal{X}$ a domain. The element $x\in\calX$ is the point. The concept $c$ is a function $c:\mathcal{X}\rightarrow\{0,1\}$. The concept class $\mathcal{C}$ is a set of concepts. We call a learned function $h$ a hypothesis. For any function $f$, we denote $I(f)=\{x:f(x)=1\}$

\section{Preliminary}
\subsection{Learning theory}

\begin{definition}[Error]
    Let $\mathcal{D}$ be a distribution, $c$ be a concept and $h$ be a hypothesis. The error of $h$ over $\mathcal{D}$ is defined as 
    $$
    error_{\mathcal{D}}(c,h)=\Pr_{x\sim\mathcal{D}}[c(x)\neq h(x)].
    $$
    For a finite set $S$, the error of $h$ over $S$ is defined as 
    $$
    error_S(c,h)=\frac{\left|\{x\in S:c(x)\neq h(x)\}\right|}{|S|}.
    $$
    We write it as $error_S(h)$ because $c(x)$ is given as the label of $x$ in $S$.
\end{definition}

\begin{definition}[PAC Learning~\cite{Valiant84}]
    Given a set of $n$ points $S=(\mathcal{X}\times\{0,1\})^n$ sampled from distribution $\mathcal{D}$ and labels that are given by an underlying concept $c$, we say a learner $L$ ($L$ could be randomized) $(\alpha,\beta)$-PAC learns $\mathcal{C}$ if $h=L(S)$ and 
    $$
    \Pr[error_{\mathcal{D}}(c,h)\leq\alpha]\geq1-\beta.
    $$
\end{definition}

\begin{definition} [Proper Learning]
 For the concept class $\mathcal{C}$, we call the learner proper if the output hypothesis $h\in\mathcal{C}$. We call the learner improper if the output hypothesis $h\notin\mathcal{C}$.
\end{definition}

\begin{definition}[VC Dimension\cite{VC}]
    For a domain $\mathcal{X}$ and a \emph{concept class} $\mathcal{C}$, We say $x_1,\dots,x_k$ is \emph{shattered} if for any subset $S\subseteq \{x_1,\dots,x_k\}$, there exists a concept $c\in\mathcal{C}$, such that $c(x)=1$ for any $x\in S$ and $c(x)=0$ for any $x\in \{x_1,\dots,x_k\}\backslash S$. The maximum size of the shattered set is called \emph{VC dimension}.
\end{definition}

\begin{theorem}[\cite{BlumerEhHaWa89}]\label{thm:learn vc}
    Let $\mathcal{C}$ be a concept class, and let $\mathcal{D}$ be a distribution over the domain $\mathcal{C}$. Let $\alpha,\beta>0$, and $m\geq \frac{48}{\alpha}\left(10VC(\mathcal{X},\mathcal{C})\log(\frac{48e}{\alpha})+\log(\frac{5}{\beta})\right)$. Let $S$ be a sample of $m$ points drawn i.i.d.\ from $\mathcal{D}$. 
    Then
    $$\Pr[\exists c,h\in\mathcal{C}~\mbox{s.t.}~error_{S}(c,h)\leq\alpha/10 ~\mbox{and}~error_{\mathcal{D}}(c,h)\geq\alpha]\leq \beta.$$
\end{theorem}

\begin{definition} [Thresholds Dimension]
    For $x_1,\dots,x_k\in\calX$ and $c_1,\dots,c_k\in\mathcal{C}$, if for any $i\in[k]$, we have $c_i(x_j)=1$ for all $j\geq i$ and $c_i(x_j)=0$ for all $j<i$, then we call $\left((x_1,\dots,x_k),(c_1,\dots,c_k)\right)$ as a class of thresholds. The \emph{thresholds dimension} $TD(\calX,\mathcal{C})=argmax_k\{\exists x_1,\dots,x_k\in\calX,c_1,\dots,c_k\in\mathcal{C},\left((x_1,\dots,x_k),(c_1,\dots,c_k)\right)~\mbox{is a thresholds class}\}$, i.e. the length of the longest thresholds in $(\calX,\mathcal{C})$.
\end{definition}

\begin{definition}[Online Learning~\cite{littlestone87}]
    In the $i^{th}$ turn of the online learning setting, the learner receives a data point $x_i$ and predicts the label of $x_i$. Then the learner receives the true label of $x_i$.
\end{definition}

\begin{definition}[Littlestone Dimension~\cite{littlestone87}]
    In online learning, we say the learner makes a mistake if the learner's prediction is different from the true label in one turn. We say a learner is optimal if the learner can make the minimum number of mistakes when the learner outputs the true concept. The maximum number of mistakes the optimal learner makes is called the Littlestone dimension of $(\calX,\mathcal{C})$. We denote it as $d_L(\calX,\mathcal{C})$.
\end{definition}

\begin{theorem}~\cite{shelah1990classification,hodges1997shorter,AlonLMM19}
    $\lfloor\log d_L(\mathcal{X},\mathcal{C}))\rfloor\leq TD(\mathcal{X},\mathcal{C})\leq 2^{d_L(\mathcal{X},\mathcal{C})+1}$.
\end{theorem}

\begin{corollary}\label{cor: threshold dimension and littlestone dimension}
    $O(\log^*(TD(\mathcal{X},\mathcal{C})))=O(\log^*(d_L(\mathcal{X},\mathcal{C})))$
\end{corollary}

\subsection{Differential privacy}

\begin{definition}[Differential Privacy~\cite{DMNS06}]
    A mechanism $M$ is called $(\varepsilon,\delta)$-differentially private if for any two dataset $S$ and $S'$ that differ on one entry and any event $E$, it holds that
    $$
    \Pr[M(S)\in E]\leq e^{\varepsilon}\cdot \Pr[M(S')\in E]+\delta.
    $$
\end{definition}

\begin{theorem}[Post-processing~\cite{DMNS06}]\label{thm:post-processing}
    For any $(\varepsilon,\delta)$-differentially private mechanism $M$ and any function $A$ ($A$ could be randomized), the mechanism $A\circ M$ is $(\varepsilon,\delta)$-differentially private.
\end{theorem}

\begin{theorem}[Composition~\cite{DworkL09}]\label{thm:composition}
    For an $(\varepsilon_1,\delta_1)$-differentially private mechanism $M_1$ and an $(\varepsilon_2,\delta_2)$-differentially private mechanism $M_2$, the composed mechanism $M(X)=(M_1(X),M_2(X))$ is $(\varepsilon_1+\varepsilon_2,\delta_1+\delta_2)$-differentially private.
\end{theorem}

\begin{theorem}[Advanced composition~\cite{DRV10}]\label{thm:advancedcomposition}
    Let $M_1,\dots,M_k:\calX^n\rightarrow \calY$ be $(\varepsilon,\delta)$-differentially private mechanisms. Then the algorithm that on input $S\in \calX^n$ outputs $(M_1(S),\dots,M_k(S))$ is $(\varepsilon',k\delta+\delta')$-differentially private, where $\varepsilon'=\sqrt{2k\ln(1/\delta')}\cdot \varepsilon$ for every $\delta'>0$.
\end{theorem}

\begin{theorem}[Exponential Mechanism~\cite{McSherryT07}]\label{thm:exponential mechanism}
Give a score function $q(D,r)$, where $D$ is a database and $r\in R$ is a possible output. Let $\Delta=\max\left(|q(x,r)-q(x',r)|\right)$ for all $r$ and neighbouring databases $x$ and $x'$. The mechanism $M(D)$ outputs $r$ 
with probability$\frac{\mbox{exp}(\frac{\varepsilon q(D,r)}{2\Delta})}{\sum_{i\in R}\mbox{exp}(\frac{\varepsilon q(D,i)}{2\Delta})}$. Then $M$ is $\varepsilon$-differentially private.
\end{theorem}

\begin{definition}[Differentially Private PAC Learning~\cite{KLNRS08}]
    Given a set of $n$ points $S=(\mathcal{X}\times\{0,1\})^n$ sampled from distribution $\mathcal{D}$ and labels that are given by an underlying concept $c$, we say a learner $L$ $(\alpha,\beta,\varepsilon,\delta)$-differentially privately PAC learns the concept class $\mathcal{C}$ if
    \begin{enumerate}
        \item $L$ $(\alpha,\beta)$-PAC learns $\mathcal{C}$.
        \item $L$ is $(\varepsilon,\delta)$-differentially private.
    \end{enumerate}
\end{definition}

\begin{definition}[$\alpha$-median]
    For a set of number $S=\{x_1,\dots,x_n\}$, we say a number $\hat{x}$ is an $\alpha$-median of $S$ if $\min\{|\{x:x\leq\hat{x},x\in S|,|\{x:x\geq\hat{x},x\in S|\}\geq (1/2-\alpha)\cdot|S|$
\end{definition}

\begin{fact}
    For any set of number $S=\{x_1,\dots,x_n\}$, there exists a $1/2$-median.
\end{fact}

\begin{theorem}[\cite{CohenLNSS22}]~\label{thm:private median}
    Let $\calX$ be a finite ordered domain. 
    There exists an $(\varepsilon,\delta)$-differentially private algorithm $\textsf{PrivateMedian}$\footnote{In \cite{CohenLNSS22}, they provide an algorithm that can privately select interior point. That is given $x_1,\dots,x_k$, the algorithm privately outputs a number $\hat{x}$ satisfying $\min\{x_i\}\leq \hat{x}\leq\max\{x_i\}$. It can be extended to the median by removing the smallest and largest $(1/2-\alpha/2)$ fraction of numbers. This reduction is found by Bun et al.~\cite{BNSV15}.} that on input $S\in \calX^n$ outputs an $\alpha$-median point with probability $1-\beta$ provided that $n > n_{PM}(|\calX|,\alpha,\beta,\varepsilon,\delta)$ for $n_{PM}(|\calX|,\alpha,\beta,\varepsilon,\delta) \in O\left(\frac{\log^*|\calX|\cdot\log^2(\frac{\log^*|\calX|}{\beta\delta})}{\alpha\varepsilon}\right)
    $.
\end{theorem}

\begin{definition}[$k$-bounded function]\label{def: k-bounded}
   We call a quality function $q:X^*\times Z\rightarrow\mathbb{R}$ is $k$-bounded if adding a new element to the data set can only increase the score of at most $k$ solutions, Specifically, it holds that
\begin{enumerate}
    \item $q(\emptyset,z)=0$ for every $z\in Z$.
    \item If $D'=D\cup\{x\}$, then $q(D',z)\in\{q(D,z),q(D,z)+1\}$ for every $z\in Z$, and
    \item There are at most $k$ solutions $z$ such that $q(D',z)=q(D,z)+1$
\end{enumerate} 
\end{definition}

\begin{lemma}[Choosing Mechanism~\cite{BNS13b}]~\label{lem:choosing mechanism}
    Let $\varepsilon\in(0,2)$ and $\delta>0$. Let $q:X^*\times Z\rightarrow\mathbb{R}$ be a $k$-bounded quality function. There is an $(\varepsilon,\delta)$-DP algorithm $\mathcal{A}$, such that given a dataset $D\in X^n$, $\mathcal{A}$ outputs a solution $z$ and
    $$
    \Pr[q(D,z)]\geq\max_{z\in Z}\{q(D,z)\}-\frac{16}{\varepsilon}\log(\frac{4kn}{\beta\varepsilon\delta})]\geq 1-\beta
    $$
\end{lemma}

\section{Structure of Classes with VC Dimension 1}
Without loss of generality, we have the following two assumptions. So that every point in the domain $\mathcal{X}$ is different and every concept in $\mathcal{C}$ is different. 
\begin{assumption}\label{assumption: distinguish points}
    Assume for any two different points $x_1,x_2\in\calX$, there exists a concept $c\in\mathcal{C}$ makes $c(x_1)\neq c(x_2)$. Otherwise, we can replace all $x_2$ by $x_1$ in the given dataset and remove $x_2$ from $\mathcal{X}$.
\end{assumption}

\begin{assumption}\label{assumption: distinguish concepts}
    Assume for any point $x$, there exists two different $c_1,c_2\in\mathcal{C}$ make $c_1(x)\neq c_2(x)$. Otherwise, we can remove $c_2$ from $\mathcal{C}$.
\end{assumption}

\begin{definition}[Partial Order]\label{def:order}
    Given $\calX,\mathcal{C}$ and $x_1,x_2\in\calX$, we say $x_1\preceq x_2$ under $\mathcal{C}$ if for all $c\in\mathcal{C}$, $(c(x_1)=1)\Rightarrow (c(x_2)=1)$
\end{definition}

\begin{example}
For thresholds $\left((x_1,\dots,x_k),(c_1,\dots,c_k)\right)$, we have $x_1\preceq\dots\preceq x_k$.
\end{example}

We say $x_1$ and $x_2$ are comparable under $\mathcal{C}$ if $x_1\preceq x_2$ or $x_2\preceq x_1$ under $\mathcal{C}$.

\remove{
\begin{fact}[Reflexivity]
    For any $x\in\mathcal{X}$, $x\preceq x$.
\end{fact}
\begin{fact}[Antisymmetry]
    If $x_1\preceq x_2$ and $x_2\preceq x_1$, then $x_1=x_2$.
\end{fact}
\begin{fact}[Transitivity]
    If $x_1\preceq x_2$ and $x_2\preceq x_3$, then $x_1\preceq x_3$.
    
\end{fact}
}

\begin{definition}[$f$-representation]\label{def:f representation}
    For function $c,f:\mathcal{X}\rightarrow\{0,1\}$, the \emph{$f$-representation} of $c$ is 
    $$
    c_f(x)=\left\{\begin{array}{rl}
        1 &  f(x)\neq c(x)\\
        0 &  f(x)= c(x)
    \end{array}\right.
    $$ 
    For the class of function $\mathcal{C}$, the \emph{$f$-representation} of $\mathcal{C}$ is $\mathcal{C}_f=\{c_f:c\in \mathcal{C}\}$.
\end{definition}

Given $f$ and $c_f$, we can transform $c_f$ to $f$:
$$
c(x)=\left\{\begin{array}{rl}
        1 &  f(x)\neq c_f(x)\\
        0 &  f(x)= c_f(x)
    \end{array}\right.
$$
Given the pair of the point and label $(x,c(x))$ that is labeled by $c$, we can transform it to a corresponding pair labeled by $c_f$: let label be 1 if $c(x)\neq f(x)$ and be $0$ if $c(x)= f(x)$.

\begin{lemma}    $VC(\calX,\mathcal{C})=VC(\calX,\mathcal{C}_f)$ and $d_L(\calX,\mathcal{C})=d_L(\calX,\mathcal{C}_f)$ for any $f$.
\end{lemma}
\begin{proof}
    Since $\mathcal{C}$ is the $f$-representation of $\mathcal{C}_f$, we only consider one direction.

    \begin{enumerate}
        \item \textbf{VC dimension}. Let $x_1,\dots,x_{VC(\calX,\mathcal{C})}$ be a set of points shattered by $\mathcal{C}$. For any set of dichotomy $\left(c(x_1),\dots.c(x_{VC(\calX,\mathcal{C})})\right)$, there exists a concept $c'\in\mathcal{C}$ makes  $\left(c'(x_1),\dots.c'(x_{VC(\calX,\mathcal{C})})\right)=\left((c_f(x_1),\dots.c_f(x_{VC(\calX,\mathcal{C})})\right)$ because $x_1,\dots,x_{VC(\calX,\mathcal{C})}$ are shattered. Thus, the corresponding $c'_f\in\mathcal{C}_f$ makes $\left(c'_f(x_1),\dots.c'_f(x_{VC(\calX,\mathcal{C})})\right)=\left((c(x_1),\dots.c(x_{VC(\calX,\mathcal{C})})\right)$. So that all $2^{VC(\calX,\mathcal{C})}$ dichotomies can be labeled by concepts of $\mathcal{C}_f$, which implies $x_1,\dots,x_{VC(\calX,\mathcal{C})}$ can be shattered by $\mathcal{C}_f$. Therefore $VC(\calX,\mathcal{C})\leq VC(\calX,\mathcal{C}_f)$.

        \item \textbf{Littlestone dimension}. Assume for $\mathcal{C}$, there is an optimal function $\mathcal{O}:(\mathcal{X}\times\{0,1\})^*\times\mathcal{X}\rightarrow\{0,1\}$\footnote{Littlestone~\cite{littlestone87} provides a general method to achieve the optimal number of mistakes called standard optimal algorithm (SOA).} that receives pairs of points and labels and one new point and outputs a prediction label of the new point. Then, we can construct the corresponding $\mathcal{O}_f$: 
        \begin{enumerate}
            \item for any pair of $(x,c_f(x))$, if $c_f(x)=1$, set the label to $1-f(x)$, if  $c_f(x)=0$, set the label to $f(x)$.

            \item feed all pair of points and labels and the new point $x_{new}$ to $\mathcal{O}$, receive a label $y$.

            \item output $1$ if $y\neq f(x_{new})$, otherwise output 0.
        \end{enumerate}

        So that for the concept class $\mathcal{C}_f$, the number of mistakes made by $\mathcal{O}_f$ is at most $d_L(\mathcal{X},\mathcal{C})$, which implies $d_L(\mathcal{X},\mathcal{C}_f)\leq d_L(\mathcal{X},\mathcal{C})$.
    \end{enumerate}
\end{proof}

In the remaining part of this paper, we only consider $f\in \mathcal{C}$.

\begin{observation}
    When $f\in\mathcal{C}$, there exists a function $c\in \mathcal{H}_f$, such that $c(x)=0$ for all $x\in\mathcal{C}$
\end{observation}
\begin{proof}
    $f_f$ is such a function.
\end{proof}

\begin{lemma}~\cite{BenDavid20152NO}\label{lem: must comparable}
    When $f\in\mathcal{C}$, if $x_1$ and $x_2$ are incomparable under $\mathcal{C}_f$, then there is no $c\in \mathcal{C}_f$, such that $c(x_1)=1$ and $c(x_2)=1$.
\end{lemma}
\begin{proof}
    If $x_1$ and $x_2$ are incomparable, then there exists $c_1$ makes $c_1(x_1)=1$ and $c_1(x_2)=0$ (and $c_2$ makes $c_2(x_1)=0$ and $c_2(x_2)=1$, respectively). If there is $c\in \mathcal{C}_f$, such that $c(x_1)=1$ and $c(x_2)=1$, then $x_1,x_2$ are shattered because $f_f\in \mathcal{C}_f$. It makes the VC dimension at least 2.
\end{proof}

\subsection{Tree Structure}
Then, we can build the tree structure of $\mathcal{C}_f$ according to the partial order relationship. Specifically, for any point $x$, if there is no $x'$ makes $x\prec x'$, we define $x\prec\emptyset$.

\begin{algorithm}~\label{alg:maketree}
    \caption{MakeTree}
    /*In this algorithm, we call a node a leaf if the node does not have children.*/
    
    \textbf{Inputs:} a concept class $\mathcal{C}$ with VC dimension 1
    \begin{enumerate}
        \item Select a function $f\in \mathcal{C}$, construct $\mathcal{C}_f$ according to Definition~\ref{def:f representation}. In this algorithm, all the partial order relationships are under $\mathcal{C}_f$.

        \item add $\emptyset$ to the tree $T$.

        \item If there is $x\in\calX$ and $x\notin T$:
        \begin{enumerate}
            \item  Let $L$ be the set of leaves of the tree. For $\ell\in L$:

            \begin{enumerate}
            \item  select all the points $x'$ satisfying $x'\prec \ell$ and there is no $x''$ makes $x'\prec x''\prec x$. Then make $x'$ to be the child of $\ell$.
            
        \end{enumerate}
            
        \end{enumerate}
        \item Output the tree $T$.

    \end{enumerate}
\end{algorithm}

\begin{definition}[Deterministic Point]
    For a point $x$ and a labeled dataset $S$, we say $x$ is deterministic by $S$  in $\mathcal{C}$ if, for all $h\in\{c\in\mathcal{C}:error_S(c)=0\}$ (that is all concepts that agrees with $S$), we have $h(x)=1$.
\end{definition}

\begin{definition}[Distance]
    For an ordered sequence $x\prec x_1\prec\dots\prec x_k\prec\emptyset$ under $\mathcal{C}_f$. We define $d_{\mathcal{C}_f}(x)=\max{k}+1$ as the distance of a point $x$ in $\mathcal{C}_f$. Specifically, $d(\emptyset)=0$.
\end{definition}

\begin{lemma}
    $\max_{x\in\mathcal{X}}d_{\mathcal{C}_f}(x)\leq TD(\mathcal{X},\mathcal{C}_f)$
\end{lemma}
\begin{proof}
    It is equivalent to show for any $x_0\prec x_1\dots\prec x_k\prec\emptyset$, there exist corresponding $c_0,\dots,c_k\in \mathcal{C}_f$ make $c_i(x_j)=1$ if and only if $i\leq j$.

    We first select $c_0$. There must exist a $c_0\in\mathcal{C}_f$ makes $c_0(x_0)=1$, otherwise all concepts $c_f\in\mathcal{C}_f$ make $c_f(x_0)=0$, which makes corresponding $c(x_0)=f(x_0)$ for all $c\in\mathcal{C}$. It contradicts Assumption~\ref{assumption: distinguish concepts}. By Definition~\ref{def:order}, $c_0(x_i)=1$ for all $i\geq0$.
    
    Assume we already find $c_0,\dots,c_{k'}$, by Assumtion~\ref{assumption: distinguish points}, there exist a concept $c'$ makes $c'(x_{k'})\neq c'(x_{k'+1})$. Since $c(x_{k'})=1\Rightarrow c(x_{k'+1})=1$, it must be that $c'(x_{k'})=0$ and $c'(x_{k'+1})=1$. By Definition~\ref{def:order}, $c'(x_i)=1$ for all $i\geq k'+1$ and we can set $c'$ to be $c_{k'+1}$.

    At the end, we have $f_f(x_i)=0$ for all $i$.
\end{proof}

\section{Improper learner}

\begin{algorithm}\label{alg:improperlearner}
    \caption{OPTPrivateLearner}

    \textbf{Parameter:} Confidence parameter $\beta > 0$, accuracy paramenter $\alpha>0$, privacy parameter $\varepsilon,\delta > 0$, $n_{PM}(d,1/3,\beta,\varepsilon,\delta)=O\left(\frac{\log^*d\cdot\log^2(\frac{\log^*d}{\beta\delta})}{\varepsilon}\right)$ and number of subsets $t=\max\left\{n_{PM}(d,1/3,\beta,\varepsilon,\delta),O\left(\frac{1}{\varepsilon}\log(\frac{4n_{PM}(d,1/3,\beta,\varepsilon,\delta)}{\beta\varepsilon\delta})\right)\right\}=O\left(\frac{\log^*d\cdot\log^2(\frac{\log^*d}{\varepsilon\beta\delta})}{\varepsilon}\right)$, where $d=TD(\calX,\mathcal{C}_f)+1$.

    \textbf{Inputs:} Labeled dataset $S\in (\calX\times\{0,1\})^N$, where $N= t\cdot \frac{48}{\alpha}\left(10\log(\frac{48e}{\alpha})+\log(\frac{5}{\beta})\right)$

    \textbf{Operation:}~

    \begin{enumerate}
        \item construct tree according to Algorithm~\ref{alg:maketree}, let $f$ be the select concept in Algorithm~\ref{alg:maketree}, transfer all labels in $S$ according to $f$ (by Definition~\ref{def:f representation}).

        \item randomly partition $S$ into $S_1,\dots,S_t$

        \item For $i\in[t]$
        \begin{enumerate}
            \item Let $B_i$ be the set of points deterministic by $S_i$ in $\mathcal{C}_f$. Let $y_i=\max_{x\in B_i}d(x)$, that is the largest distance of a point in $B_i$.
        \end{enumerate}

        \item\label{step:median-distance} compute the $1/3$-median $z=PrivateMedian(y_1,\dots,y_t)$ with parameter $\varepsilon,\delta,\beta$.

        \item let $P=\{x:d_{\mathcal{C}_f}(x)=z\}$, i.e. the set of points with distance $z$. Define $P=\{x_1,\dots,x_m\}$ and $q_1=q_2=\dots=q_{m}=0$
        
        \item For $i\in[t]$:
        \begin{enumerate}
            \item if $y_i\geq z$:
            \begin{enumerate}
                \item for $j\in[m]$, if $x_j\in B_i$, make $q_j=q_j+1$.
            \end{enumerate}
            \item if $y_i< z$, do nothing.
        \end{enumerate}

        \item\label{step:choose good x} run choosing mechanism on $(q_1,q_2,\dots,q_{m})$ with parameter $\varepsilon,\delta,\beta$, select $p_{good}$ and the corresponding point $x_{good}$.

        \item let $I_{good}=\{x|x_{good}\preceq x\}$. Construct $\hat{I}_{good}=\{x:(x\in I_{good}\wedge   f(x)=0)\vee (x\notin I_{good}\wedge   f(x)=1)\}$.
        
        \item construct and output 
        $$
        \hat{h}_{good}(x)=\left\{
        \begin{array}{rl}
            1 & x\in\hat{I}_{good} \\
            0 & x\notin\hat{I}_{good}
        \end{array}
        \right.
        $$
    \end{enumerate}    
\end{algorithm}

\begin{theorem}~\label{thm:privacy-OPTPrivateLearner}
    \textsf{OPTPrivateLearner} is $(2\varepsilon,2\delta)$-differentially private.
\end{theorem}
\begin{proof}
For each different entry of $S$, there is at most one different element in $y_i,\dots,y_t$ and $q_1,\dots,q_t$. Thus by Theorem~\ref{thm:private median}, $z$ is $(\varepsilon,\delta)$-differentially private in Step~\ref{step:median-distance}. Notice that $q_1,\dots,q_m$ is 1-bouned function (Definition~\ref{def: k-bounded}). By Lemma~\ref{lem:choosing mechanism}, $x_{good}$ is $(\varepsilon,\delta)$-differentially private in Step~\ref{step:choose good x}. By the composition (Theorem~\ref{thm:composition}) and post-processing (Theorem~\ref{thm:post-processing}), $\hat{h}_{good}$ is $(2\varepsilon,2\delta)$-differentially private. 
\end{proof}

\begin{lemma}~\cite{BenDavid20152NO}~\label{lem:order in h}
    For every $c\in \mathcal{C}_f$ and the corresponding set $I(c)=\{x_1,\dots,x_r\}$, the following two statements are true:
    \begin{enumerate}
        \item There is an order $\pi(1),\pi(2),\dots,\pi(r)$ to make $x_{\pi(1)}\prec\dots\prec x_{\pi(r)}$. 
        \item There is no $\hat{x}\in \calX$ to make $x_{\pi(r)}\prec \hat{x}$.
    \end{enumerate}
\end{lemma}
\begin{proof}
    By Lemma~\ref{lem: must comparable}, every $x,x'\in I(c)$ are comparable. Sort all points in $I(c)$ by their distances and it is the order required.  
    
    There is no $\hat{x}\in \calX$ to make $x_{\pi(r)}\prec \hat{x}$ because if there is a $x_{\pi(r)}\prec \hat{x}$, by Definition~\ref{def:order}, we have $\hat{x}\in I(c)$. Then $x_{\pi(r)}$ is not the last point in this order sequence.
\end{proof}

\begin{lemma}~\label{lem:determined point in c}
    For every deterministic point $x$, we have $x\in I(c_f^*)$, where $c^*$ is the underlying true concept and $c_f^*$ is the $f$-representation of  $c^*$.
\end{lemma}
\begin{proof}
    For a dataset $S$, a point $x$ is deterministic if $c(x)=1$ for any $c$ with $error_S(c,c^*_f)=0$. The lemma can be concluded by substituting $c$ with $c^*_f$.
\end{proof}

\begin{lemma}
    Let $h_i$ be the hypothesis with $I(h_i)=B_i$, then with probability $1-\beta t$, all $h_i$ are $\alpha$-good hypothesis.
\end{lemma}
\begin{proof}
    Let $c^*_f$ be the $f$-representation of the underlying true concept. If $h_i=c^*_f$, then we are done.  Otherwise, note that $y_i$ is the point with the largest distance in $B_i$. It means for any $x\prec y_i$ with $c^*_f(x)=1$, there exists one $h'\neq c^*_f$ makes $h'(x)=0$ and $error_{S_i}(h',c^*_f)=0$ (otherwise $x$ is also deterministic, but $d_{\mathcal{C}_f}(x)>d_{\mathcal{C}_f}(y_i)$, contradicting to $y_i$ is the point with largest distance). By Theorem~\ref{thm:learn vc}, with probability $1-\beta$, $h_i$ is an $\alpha$-good hypothesis. 
    
    Notice that $B_i\subseteq I(h')$. Consider the set $I(h')\backslash B_i$. For all $x\in (h_1\backslash B_i)$, we have $c^*_f(x)=0$ because $x\notin I(c^*_f)$. Therefore $error_{\mathcal{D}}(h_i,c^*_f)\leq error_{\mathcal{D}}(h',c^*_f)\leq\alpha$ (because for all points that $h'$ and $h_i$ make different predictions, $h_i$ gives the correct label).

    Finally, the lemma can be concluded by union bound.
\end{proof}

\begin{lemma}~\label{lem:order of y}
    For all $y_1,\dots,y_t$, there is an order $\pi(1),\pi(2),\dots,\pi(t)$ to make $y_{\pi(1)}\preceq\dots\preceq y_{\pi(t)}$.
\end{lemma}
\begin{proof}
    By Lemma~\ref{lem:determined point in c}, all $y_i\in I(c^*_f)$. By Lemma~\ref{lem:order in h},  there is an order $\pi(1),\pi(2),\dots,\pi(t)$ to make $y_{\pi(1)}\preceq\dots\preceq y_{\pi(t)}$ (here it is possible to have $y_i=y_j$ for $i\neq j$).
\end{proof}

\begin{lemma}~\label{lem:accuracy h good}
    Let $h_i$ be the hypothesis with $I(h_i)=B_i$ and $h_{good}$ be the hypothesis with $I(h_{good})=I_{good}$, when  all $h_i$ are $\alpha$-good hypothesis, with probability $1-2\beta$, we have $error_{\mathcal{D}}(h_{good},c^*_f)\leq\alpha$.
\end{lemma}
\begin{proof}
    Let $\pi(1),\pi(2),\dots,\pi(t)$ be the order in Lemma~\ref{lem:order of y}. By Theorem~\ref{thm:private median}, with probability $1-\beta$, there are at least $t/6$ $y_i$'s make $y_i\preceq x_{good}$. It means $\sum_{i}^mq_i\geq t/6$.

    We claim that every different point in $S$ makes at most one $q_i$ different. Otherwise, there exist different $x,x'$ with distance $z$  and one $B_i$ to make $x,x'\in B_i$. By Lemma~\ref{lem: must comparable}, $x$ and $x'$ are comparable. Assume $x\prec x'$, it makes $d_{\mathcal{C}_f}(x)>d_{\mathcal{C}_f}(x')$.
        
    So that we can apply the choosing mechanism. For any $x\neq x_{good}$, they will get a 0 score. For $x_{good}$, it will get a score of at least $t/6\geq \frac{16}{\varepsilon}\log(\frac{4n_{PM}(d,1/3,\beta,\varepsilon,\delta)}{\beta\varepsilon\delta})$. By Lemma~\ref{lem:choosing mechanism}, with probability at least $1-\beta$, choosing mechanism outputs $x_{good}$.
    
    Since there is at least one $y_i$  make $x_{good}\preceq y_i$, we have $B_i\subseteq I_{good}$. Thus $ error_{\mathcal{D}}(h_{good},c^*_f)\leq error_{\mathcal{D}}(h_i,c^*_f)\leq\alpha$ because for all points that $h_{good}$ and $h_i$ make different predictions, $h_{good}$ gives the correct label
\end{proof}

\begin{corollary}
    With probability $1-(t+2)\beta$, $\textsf{OPTPrivateLearner}$ outputs $\hat{h}_{good}$ satisfying $error_{\mathcal{D}}(\hat{h}_{good},c^*)\leq\alpha$.
\end{corollary}
\begin{proof}
    The accuracy is because $error_{\mathcal{D}}(\hat{h}_{good},c^*)=error_{\mathcal{D}}(h_{good},c^*_f)$. The confidence is by the union bound.
\end{proof}

Substitute $\varepsilon$ by $\varepsilon/2$, $\delta$ by $\delta/2$, and $\beta$ by $\beta/(t+2)$, and considering Corollary~\ref{cor: threshold dimension and littlestone dimension}, we have the main result.

\begin{theorem}~\label{thm:improper sample complexity}
    For any concept class $\mathcal{C}$ with VC dimension 1 and Littlestone dimension $d$, and given labeled dataset with size
    $$
    N\geq O\left(\frac{\log^*d\cdot\log^2(\frac{\log^*d}{\varepsilon\beta\delta})}{\varepsilon}\cdot \frac{48}{\alpha}\left(10\log(\frac{48e}{\alpha})+\log(\frac{5}{\beta})\right)\right)=\tilde{O}_{\beta,\delta}\left(\frac{\log^*d}{\alpha\varepsilon}\right)
    $$
    there is an $(\varepsilon,\delta)$-differentially private algorithm that $(\alpha,\beta)$-PAC learns $\mathcal{C}$.
\end{theorem}

\section{Proper Learner}
The output hypothesis $\hat{h}_{good}$ is improper when the concept class is not maximum. In the improper learner, it makes the output hypothesis misclassify at most $O(\alpha)$ fraction (on the probability of the underlying distribution) of points with label 1 (on the tree $T$) with high probability. Thus, in the remaining part, we can only consider the points with label 0.

\subsection{Construction of the Proper Learner}
We still consider the tree $T$ of the concept class. Different from the improper learner, we consider a subtree $T_{good}$ such that only the leaves of $T_{good}$ can be projected to a concept in the concept class $\mathcal{C}$.

\subsubsection{Additional Definitions and Notations}
\begin{definition}[Proper Node]
    We call a node $x$ a "proper node" if there exist $c_f\in\mathcal{C}_f$ to make $I(c_f)=\{y|y\preceq x\}$.
\end{definition}

\begin{definition}[Weight of the Node]
    Given the tree $T$ in Algorithm~\ref{alg:maketree} and dataset $S$, for any node $x$ in the tree $T$, define the weight of $x$ to be
    $$
    w(x)=|\{x'|x'\preceq x,(x',0)\in S\}|.
    $$
    That is, the number of points with label 0 in $S$ that have a lower or equal partial order than $x$.
\end{definition}

\begin{definition}[Value of the Node]
    Given the tree $T_{good}$ and it's root $x_{good}$ in Algorithm~\ref{alg:makesubtree} and dataset $S$, for any node $x$ in the tree $T_{good}$, define the value of $x$ to be
    $$
    v(x)=|\{x'|x\preceq x'\prec x_{good},(x',0)\in S\}|.
    $$
    That is, the number of points with label 0 in $S$ that are on the path from $x$ to $x_{good}$.
\end{definition}

\begin{definition}[Sub-tree with Root $x$]
    We define $T_x$ to be the sub-tree with root $x$ in $T_{good}$ constructed by Algorithm~\ref{alg:makesubtree}.
\end{definition}

\subsubsection{The Algorithm}
We first construct the subtree $T_{good}$, then construct the proper learner based on the subtree $T_{good}$.
\begin{algorithm}~\label{alg:makesubtree}
    \caption{MakeSubTree}
    /*In this algorithm, we call a node a leaf if the node does not have children.*/
    
    \textbf{Inputs:} A tree $T$ constructed by Algorithm~\ref{alg:maketree}, a point $x_{good}\in\mathcal{X}$ by Algorithm~\ref{alg:improperlearner}. 
    \begin{enumerate}
        \item If $x_{good}$ is a proper node, output $x_{good}$.
        
        \item Remove all nodes $x$ such that $x$ is incomparable to $x_{good}$ or $x_{good}\prec x$

        /*This step makes $x_{good}$ the root of the new tree.*/

        \item Remove all nodes $x$ such that:
        \begin{enumerate}
            \item There exists $x'$ such that $x\prec x' \prec x_{good}$ and 
            \item $x'$ is a proper node.
        \end{enumerate}

        /*This step guarantees that the node of the new tree can be a proper node if and only if the node is the leaf of the new tree.*/
        
        \item Output the new tree $T_{good}$.

    \end{enumerate}
\end{algorithm}

\begin{algorithm}~\label{alg:properlearner}
    \caption{ProperLearner}

    \textbf{Parameter:} Confidence parameter $\beta > 0$, accuracy paramenter $\alpha>0$, privacy parameter $\varepsilon,\delta > 0$

    \textbf{Inputs:} Labeled dataset $S\in (\calX\times\{0,1\})^{N_1+N_2}$, where $N_1$ is the sample size in Theorem~\ref{thm:improper sample complexity}, $N_2=O(\frac{\log(1/\alpha)+\log(1/\beta)}{\alpha^2\varepsilon})$

    \textbf{Operation:}
    \begin{enumerate}
        \item Random partition $S$ to $S_1$ and $S_2$, where $|S_1|=N_1$ and $|S_2|=N_2$.
        \item Construct tree $T$ according to Algorithm~\ref{alg:maketree}, let $f$ be the select concept in Algorithm~\ref{alg:maketree}. Use $S_1$ to  compute $x_{good}$ according to Algorithm~\ref{alg:improperlearner}.  Transfer all labels in $S$ according to $f$ (by Definition~\ref{def:f representation})

        /*Recall that $x_{good}$ can project to an $\alpha$-accurate hypothesis*/

        \item If $x_{good}$ is a proper node, output $\hat{h}_{good}$ according to Algorithm~\ref{alg:improperlearner}.

        \item Compute the weight of all nodes in the tree $T$.
        
        \item Construct the sub tree $T_{good}$ with root $x_{good}$ according to Algorithm~\ref{alg:makesubtree} and compute the value of all nodes in $T_{good}$.

        \item Let $x_{flag}=x_{good}$

        \item For $i=1,2\dots,T$, where $T=\frac{2}{\alpha}$~\label{step:loop-proper}:

        \begin{enumerate}
            \item Let $X_{children}$ be the set of child nodes of $x_{flag}$.
        
            \item Let $w_{min}=\min_{x\in X_{children}}[w(x)]$ and $\tilde{w}_{min}=w_{min}+\mbox{Laplace}(\frac{1}{\varepsilon})$

            \item \label{step:nonuniform case}If $\tilde{w}_{min}\leq \alpha\cdot N$:
            \begin{enumerate}
                \item Use exponential mechanism to select $x_{new}$ from $X_{children}$ using quality function $q(x)=-w(x)$ with privacy parameter $\varepsilon$.

                \item Set set $x_{flag}=x_{new}$, stop the loop and go to Step~\ref{step:select good leaf}
            \end{enumerate}

            \item \label{step:uniform case} If $\tilde{w}_{min}> \alpha\cdot N$:
            \begin{enumerate}
                \item Let $L(x)$ be the set of leaves of the sub-tree $T_x$, let quality function be $q(x)=-\min_{x'\in L(x)}[v(x')]$ and use exponential mechanism select $x_{new}$ from $X_{children}$ with privacy parameter $\varepsilon$

                \item Set $x_{flag}=x_{new}$
            \end{enumerate}
            
        \end{enumerate}

        \item~\label{step:select good leaf}Select any leaf $\ell$ in the tree $T_{x_{flag}}$
        \item Let $I_{good}=\{x|\ell\preceq x\}$. Construct $\hat{I}_{good}=\{x:(x\in I_{good}\wedge   f(x)=0)\vee (x\notin I_{good}\wedge   f(x)=1)\}$.
        
        \item Construct and output 
        $$
        \hat{h}_{good}(x)=\left\{
        \begin{array}{rl}
            1 & x\in\hat{I}_{good} \\
            0 & x\notin\hat{I}_{good}
        \end{array}
        \right.
        $$

    \end{enumerate}

\end{algorithm}

\subsection{Analysis}
\subsubsection{Privacy}
When the different datum is in $S_1$, the privacy can be guaranteed by Theorem ~\ref{thm:improper sample complexity}. So we only consider that the different datum is in $S_2$.
\begin{theorem}~\label{thm:privacy-OPTPrivateLearner}
    \textsf{OPTPrivateLearner} is $(\sqrt{2T\ln(1/\delta')}\cdot2\varepsilon,\delta')$-differentially private for any small $\delta'$.
\end{theorem}
\begin{proof}
The algorithm loses privacy budget in Step~\ref{step:loop-proper}. The result can be implied by advanced composition (Theorem~\ref{thm:advancedcomposition}).
\end{proof}

\subsubsection{Accuracy}
We say a node $x$ represents an $\alpha$-accuracy hypothesis if there exists a hypothesis $h$ such that $I(h)=\{x':x\preceq x'\}$ and $\Pr_{x\sim \mathcal{D}}[h(x)\neq c_f(x)]\leq \alpha$.
\begin{claim}
    Conditioned on that $x_{good}$ represents an $\alpha$-accuracy hypothesis, $\Pr\left[|x:x\prec x_{good},(x,1)\in S_f\}|>2\alpha\cdot N\right]\leq e^{-\frac{\alpha\cdot N}{3}}<\beta$, that is, with probability at most $\beta$, there are $2\alpha\cdot N$ points with label 1 in $S_2$ locate below $x_{good}$
\end{claim}
\begin{proof}
This can be proven using the Chernoff bound.
\end{proof}

\begin{claim}
    Conditioned on that $x_{good}$ represents an $\alpha$-accuracy hypothesis, there exists a leaf $\ell_{good}$ in the tree $T_{good}$ such that $\Pr[v(\ell_{good})\leq 2\alpha\cdot N]\leq e^{-\frac{\alpha\cdot N}{3}}<\beta$. That is, with probability at most $\beta$, there are $2\alpha\cdot N$ points with label 0 in $S_2$ locate on the path from $\ell_{good}$ to $x_{good}$
\end{claim}
\begin{proof}
Consider the leaf on the path from $x_c$ to $x_{good}$, where $x_c$ represents the true concept. By the Chernoff bound, with probability $1-\beta$, there are $2\alpha\cdot N$ points with label 0 in $S_2$ located on the path from $x_c$ to $x_{good}$. It also holds for all nodes on the path from $x_c$ to $x_{good}$.
\end{proof}

\begin{claim}~\label{claim:go to nonuniform case}
    With probability $1-T\beta$, Algorithm~\ref{alg:properlearner} runs Step~\ref{step:nonuniform case}.
\end{claim}
\begin{proof}
    With probability $1-T\beta$, all laplace noise in Step~\ref{step:loop-proper} less than $\frac{\ln(1/2\beta)}{\varepsilon}<\alpha\cdot N/2$, which implies $w_{min}\geq \alpha\cdot N-\frac{\ln(1/2\beta)}{\varepsilon}\geq \alpha\cdot N/2$ for all iteration. So that when the algorithm goes to the next iteration, we remove at least $\alpha\cdot N/2$ points from the total weight because they are located on other branches of the tree. Then, with probability $1-(\frac{2}{\alpha}-1)\beta$, the remaining weight less than $\alpha\cdot N/2$ after $T-1$ iterations.
\end{proof}

\begin{claim}
    When the claim~\ref{claim:go to nonuniform case} holdes, with probability $1-T\beta$, it holds that $q(x_{new})\geq -2\alpha\cdot N-T\cdot \frac{\ln (2/\alpha)}{\varepsilon}$ in Step~\ref{step:uniform case} for all iteration.
\end{claim}
\begin{proof}
    It can be proven by induction. Notice that when $\tilde{w}_{min}\geq \alpha\cdot N$, we have $w_{min}\geq \alpha\cdot N/2$, which implies that $|X_{children}|\leq \frac{2}{\alpha}$. 
    
    In the base case of $i=1$, by the guarantee of the exponential mechanism (Theorem~\ref{thm:exponential mechanism}), with probability $1-\beta$, we have $q(x_{new})\geq -2\alpha\cdot N-\frac{\ln (2/\alpha)}{\varepsilon}$. In each iteration, the quality function loses at most $\frac{\ln (2/\alpha)}{\varepsilon}$ with probability $1-\beta$. Therefore, with probability $1-T\beta$, $q(x_{new})\geq -2\alpha\cdot N-T\cdot \frac{\ln (2/\alpha)}{\varepsilon}$ for all $i$.
\end{proof}

\begin{claim}
    When $q(x_{flag})\geq -2\alpha\cdot N-T\cdot \frac{\ln (2/\alpha)}{\varepsilon}$, Step~\ref{step:nonuniform case} selects a $x_{new}$ satisfying $q(x_{new})\geq -3\alpha\cdot N-(T+1)\cdot \frac{\ln (2/\alpha)}{\varepsilon}>4\alpha\cdot N_2$ with probability at least $1-\beta/\alpha$.
\end{claim}
\begin{proof}
    $\tilde{w}_{min}\leq\alpha\cdot N$ implies $w_{min}\leq 3\alpha\cdot N/2$. Let $X_{good}=\{x:w(x)\leq 3\alpha\cdot N/2,x\in X_{children}\}$ and $X_{bad}=\{x:w(x)> 3\alpha N/2,x\in X_{children}\}$. We have $|X_{good}|\geq 1$ and $|X_{bad}|\leq \frac{2}{3\alpha}$. Let $X_{good}'=\{x:3\alpha\cdot N/2< w(x)\leq 3\alpha\cdot N,x\in X_{children}\}$ and $X_{very\;bad}=\{x:w(x)> 3\alpha N,x\in X_{children}\}$. So that we have 
    $$
    \begin{array}{rl}
         \Pr[x_{new}\in X_{very\;bad}]
         &=\frac{\sum_{x\in X_{very\;bad}}e^{\frac{-\varepsilon w(x)}{2}}}{\sum_{x\in X_{good}}e^{\frac{-\varepsilon w(x)}{2}}+\sum_{x\in X_{good}'}e^{\frac{-\varepsilon w(x)}{2}}+\sum_{x\in X_{very\;bad}}e^{\frac{-\varepsilon w(x)}{2}}}  \\
         &\leq  \frac{\sum_{x\in X_{very\;bad}}e^{\frac{-\varepsilon w(x)}{2}}}{\sum_{x\in X_{good}}e^{\frac{-\varepsilon w(x)}{2}}}\\
         &\leq \frac{\frac{2}{\alpha}\cdot e^{-3\varepsilon\alpha N/2}}{e^{-3\varepsilon\alpha N/4}}\\
         &\leq \frac{2\cdot e^{-3\varepsilon\alpha N/4}}{\alpha}\leq\beta/\alpha
    \end{array}
    $$
    Thus, we have $w(x_{new})\leq \alpha\cdot N$ with probability $1-\beta$. Notice that $q(x_{flag})\geq -2\alpha\cdot N-T\cdot \frac{\ln (2/\alpha)}{\varepsilon}$ implies that $v(x_{flag})\leq 2\alpha\cdot N+T\cdot \frac{\ln (2/\alpha)}{\varepsilon}$. So for all $\ell\in T_{x_{new}}$, we have $v(\ell)\leq v(x_{flag})+w(x_{new})\leq 2\alpha\cdot N+T\cdot \frac{\ln (2/\alpha)}{\varepsilon}+\alpha N=3\alpha\cdot N+T\cdot \frac{\ln (2/\alpha)}{\varepsilon}$.
\end{proof}

\begin{claim}
    Algorithm~\ref{alg:properlearner} is an $(O(\alpha),O(\beta/\alpha))$-PAC learner.
\end{claim}
\begin{proof}
    By the above analysis, with probability $1-O(\beta/\alpha)$, $\hat{h}_{good}$ disagrees with $O(\alpha)$ fraction labeled points in $S_2$. The result can be concluded by Theorem~\ref{thm:learn vc}.
\end{proof}

Thus, by substituting $\varepsilon$ by $\frac{\varepsilon}{\sqrt{2\ln(2/\delta)/\alpha}}$, $\delta$ by $\frac{\alpha\delta}{2}$, $\delta'$ by $\delta/2$, and $\beta$ by $\alpha\beta$, we obtain our main theorem for private proper learning.

\begin{theorem}~\label{thm:proper sample complexity}
    For any concept class $\mathcal{C}$ with VC dimension 1 and Littlestone dimension $d$, and given labeled dataset with size
    $$
    \begin{array}{rl}
        N&\geq O\left(\frac{\log^*d\cdot\log^2(\frac{\log^*d}{\varepsilon\beta\delta})}{\varepsilon}\cdot \frac{48}{\alpha}\left(10\log(\frac{48e}{\alpha})+\log(\frac{5}{\beta})\right)+\frac{(\log(1/\alpha)+\log(1/\alpha\beta))\cdot\sqrt{\log(1/\delta)}}{\alpha^{2.5}\varepsilon}\right)\\
        &=\tilde{O}_{\beta,\delta}\left(\frac{\log^*d}{\alpha\varepsilon}+\frac{1}{\alpha^{2.5}\varepsilon}\right)
    \end{array}    
    $$
    there is an $(\varepsilon,\delta)$-differentially private algorithm that properly $(\alpha,\beta)$-PAC learns $\mathcal{C}$.
\end{theorem}

\bibliographystyle{plain}
\bibliography{refs}

\begin{thebibliography}{10}

\bibitem{AlonLMM19}
Noga Alon, Roi Livni, Maryanthe Malliaris, and Shay Moran.
\newblock Private {PAC} learning implies finite littlestone dimension.
\newblock In Moses Charikar and Edith Cohen, editors, {\em Proceedings of the 51st Annual {ACM} {SIGACT} Symposium on Theory of Computing, {STOC} 2019, Phoenix, AZ, USA, June 23-26, 2019}, pages 852--860. {ACM}, 2019.

\bibitem{BKN10}
Amos Beimel, Shiva~Prasad Kasiviswanathan, and Kobbi Nissim.
\newblock Bounds on the sample complexity for private learning and private data release.
\newblock In {\em TCC}, volume 5978 of {\em LNCS}, pages 437--454. Springer, 2010.

\bibitem{beimel2019private}
Amos Beimel, Shay Moran, Kobbi Nissim, and Uri Stemmer.
\newblock Private center points and learning of halfspaces.
\newblock In {\em Conference on Learning Theory}, pages 269--282. PMLR, 2019.

\bibitem{BNS13b}
Amos Beimel, Kobbi Nissim, and Uri Stemmer.
\newblock Private learning and sanitization: Pure vs. approximate differential privacy.
\newblock In {\em APPROX-RANDOM}, pages 363--378, 2013.

\bibitem{BenDavid20152NO}
Shai Ben-David.
\newblock 2 notes on classes with vapnik-chervonenkis dimension 1.
\newblock {\em ArXiv}, abs/1507.05307, 2015.

\bibitem{BlumerEhHaWa89}
Anselm Blumer, A.~Ehrenfeucht, David Haussler, and Manfred~K. Warmuth.
\newblock Learnability and the {V}apnik-{C}hervonenkis dimension.
\newblock {\em J. ACM}, 36(4):929--965, October 1989.

\bibitem{bun2020equivalence}
Mark Bun, Roi Livni, and Shay Moran.
\newblock An equivalence between private classification and online prediction.
\newblock In {\em 2020 IEEE 61st Annual Symposium on Foundations of Computer Science (FOCS)}, pages 389--402. IEEE, 2020.

\bibitem{BNSV15}
Mark Bun, Kobbi Nissim, Uri Stemmer, and Salil~P. Vadhan.
\newblock Differentially private release and learning of threshold functions.
\newblock In {\em {FOCS}}, pages 634--649, 2015.

\bibitem{CohenLNSS22}
Edith Cohen, Xin Lyu, Jelani Nelson, Tam{\'{a}}s Sarl{\'{o}}s, and Uri Stemmer.
\newblock {\~{O}}ptimal differentially private learning of thresholds and quasi-concave optimization.
\newblock In {\em Proceedings of the 55th Annual {ACM} Symposium on Theory of Computing, {STOC} 2023}, pages 472--482, 2023.

\bibitem{DworkL09}
Cynthia Dwork and Jing Lei.
\newblock Differential privacy and robust statistics.
\newblock In {\em {STOC}}, pages 371--380. {ACM}, May 31--June 2 2009.

\bibitem{DMNS06}
Cynthia Dwork, Frank McSherry, Kobbi Nissim, and Adam Smith.
\newblock Calibrating noise to sensitivity in private data analysis.
\newblock In {\em TCC}, pages 265--284, 2006.

\bibitem{DRV10}
Cynthia Dwork, Guy~N. Rothblum, and Salil~P. Vadhan.
\newblock Boosting and differential privacy.
\newblock In {\em FOCS}, pages 51--60, 2010.

\bibitem{ghazi2021sample}
Badih Ghazi, Noah Golowich, Ravi Kumar, and Pasin Manurangsi.
\newblock Sample-efficient proper pac learning with approximate differential privacy.
\newblock In {\em Proceedings of the 53rd Annual ACM SIGACT Symposium on Theory of Computing}, pages 183--196, 2021.

\bibitem{hodges1997shorter}
Wilfrid Hodges.
\newblock {\em A shorter model theory}.
\newblock Cambridge university press, 1997.

\bibitem{KaplanLMNS19}
Haim Kaplan, Katrina Ligett, Yishay Mansour, Moni Naor, and Uri Stemmer.
\newblock Privately learning thresholds: Closing the exponential gap.
\newblock In {\em COLT}, pages 2263--2285, 2020.

\bibitem{kaplan2020private}
Haim Kaplan, Yishay Mansour, Uri Stemmer, and Eliad Tsfadia.
\newblock Private learning of halfspaces: Simplifying the construction and reducing the sample complexity.
\newblock {\em Advances in Neural Information Processing Systems}, 33:13976--13985, 2020.

\bibitem{KLNRS08}
Shiva~Prasad Kasiviswanathan, Homin~K. Lee, Kobbi Nissim, Sofya Raskhodnikova, and Adam Smith.
\newblock What can we learn privately?
\newblock {\em SIAM J. Comput.}, 40(3):793--826, 2011.

\bibitem{littlestone87}
Nick Littlestone.
\newblock Learning quickly when irrelevant attributes abound: A new linear-threshold algorithm.
\newblock In {\em 28th Annual Symposium on Foundations of Computer Science (sfcs 1987)}, pages 68--77, 1987.

\bibitem{McSherryT07}
Frank McSherry and Kunal Talwar.
\newblock Mechanism design via differential privacy.
\newblock In {\em {FOCS}}, pages 94--103. {IEEE}, Oct 20--23 2007.

\bibitem{NTY25}
Kobbi Nissim, Eliad Tsfadia, and Chao Yan.
\newblock Differentially private quasi-concave optimization: Bypassing the lower bound and application to geometric problems.
\newblock {\em arXiv preprint arXiv:2504.19001}, 2025.

\bibitem{shelah1990classification}
Saharon Shelah.
\newblock {\em Classification theory: and the number of non-isomorphic models}.
\newblock Elsevier, 1990.

\bibitem{Valiant84}
L.~G. Valiant.
\newblock A theory of the learnable.
\newblock {\em Commun. ACM}, 27(11):1134--1142, November 1984.

\bibitem{VC}
Vladimir~N. Vapnik and Alexey~Y. Chervonenkis.
\newblock On the uniform convergence of relative frequencies of events to their probabilities.
\newblock {\em Theory of Probability and its Applications}, 16(2):264--280, 1971.

\end{thebibliography}

\end{document}